\documentclass[runningheads]{llncs}
\usepackage{graphicx}
\usepackage{tikz}
\usepackage{comment}
\usepackage{amsmath,amssymb} % define this before the line numbering.
\usepackage{color}
\usepackage{orcidlink}

\usepackage[misc,geometry]{ifsym} % for corresponding author
\usepackage[accsupp]{axessibility}  % Improves PDF readability for those with disabilities.
\usepackage{multirow}
\usepackage{booktabs}
\usepackage{tabulary}
\usepackage{xspace}\xspace
\usepackage{adjustbox}
\usepackage{array}
\usepackage{subfigure}
\usepackage{tikz}
\usepackage{comment}
\usepackage{amsmath,amssymb} % define this before the line numbering.
\usepackage{color}
\def\etal{\textit{et al. }}

\newlength\savewidth\newcommand\shline{\noalign{\global\savewidth\arrayrulewidth
  \global\arrayrulewidth 1pt}\hline\noalign{\global\arrayrulewidth\savewidth}}
\newcolumntype{x}[1]{>{\centering\arraybackslash}p{#1pt}}
\newcommand{\tablestyle}[2]{\setlength{\tabcolsep}{#1}\renewcommand{\arraystretch}{#2}\centering\footnotesize}
\newcommand{\cgaphlp}[2]{
\fontsize{6pt}{-.5em}\selectfont{(${#1}$\textbf{#2})}
}
\makeatletter\renewcommand\paragraph{\@startsection{paragraph}{3}{\z@}
  {.4em \@plus1ex \@minus.2ex}{-.4em}{\normalfont\normalsize\bfseries}}\makeatother

\begin{document}
% \renewcommand\thelinenumber{\color[rgb]{0.2,0.5,0.8}\normalfont\sffamily\scriptsize\arabic{linenumber}\color[rgb]{0,0,0}}
% \renewcommand\makeLineNumber {\hss\thelinenumber\ \hspace{6mm} \rlap{\hskip\textwidth\ \hspace{6.5mm}\thelinenumber}}
% \linenumbers
\pagestyle{headings}
\mainmatter
\def\ECCVSubNumber{4047}  % Insert your submission number here

\title{PalQuant: Accelerating High-precision Networks on Low-precision Accelerators} % Replace with your title

% INITIAL SUBMISSION 
% \begin{comment}
% \titlerunning{ECCV-22 submission ID \ECCVSubNumber} 
% \authorrunning{ECCV-22 submission ID \ECCVSubNumber} 
% \author{Anonymous ECCV submission}
% \institute{Paper ID \ECCVSubNumber}
% \end{comment}
%******************

% CAMERA READY SUBMISSION
%\begin{comment}
% \titlerunning{PalQuant: Parallel Low-precision Quantization}
% If the paper title is too long for the running head, you can set
% an abbreviated paper title here
%
\author{Qinghao Hu\thanks{Equal Contribution.}\inst{1}\orcidlink{0000-0003-0422-5509
}\index{Hu, Qinghao} \and
Gang Li$^{\star}$\inst{2}\orcidlink{0000-0001-7835-4739} \and 
Qiman Wu$^{\star}$\inst{3}\orcidlink{0000-0001-7959-6860} \and Jian Cheng \inst{1}\orcidlink{0000-0003-1289-2758}\Letter}
\authorrunning{Q. Hu et al.}
% First names are abbreviated in the running head.
% If there are more than two authors, 'et al.' is used.
%
\institute{Institute of Automation, Chinese Academy of Sciences  \and
Shanghai Jiao Tong University  \and  Baidu Inc.
\\
\email{huqinghao2014@ia.ac.cn, gliaca@sjtu.edu.cn, wuqiman@baidu.com, jcheng@nlpr.ia.ac.cn}}
%\end{comment}
%******************

\maketitle

\begin{abstract}
 Recently low-precision deep learning accelerators (DLAs)  have become popular due to their advantages in chip area and energy consumption, yet the low-precision quantized models on these DLAs bring in severe accuracy degradation. 
One way to achieve both high accuracy and efficient inference is to deploy high-precision neural networks on low-precision DLAs, which is rarely studied. In this paper, we propose the PArallel Low-precision Quantization (PalQuant) method that approximates high-precision computations via learning  parallel low-precision representations from scratch. In addition, we present a novel cyclic shuffle module to boost the cross-group information communication between parallel low-precision groups. 
Extensive experiments demonstrate that PalQuant has superior performance to state-of-the-art  quantization methods in both accuracy and inference speed, e.g., for ResNet-18 network quantization, PalQuant can obtain 0.52\% higher accuracy and  1.78$\times$ speedup simultaneously over their 4-bit counter-part on a state-of-the-art 2-bit accelerator. Code is available at  \url{https://github.com/huqinghao/PalQuant}.
\end{abstract}
\section{Introduction}
Recently various model compression techniques have been proposed to deploy deep neural networks on resource-constrained edge devices. Among these, fixed-point quantization \cite{gupta2015deep,qiu2016going,rastegari2016xnor} that converts full-precision floating-point operation to low-bit integer counterpart has become the \textit{de facto} method due to its hardware efficiency.

At present, most of the commercial CNN accelerators are designed for high-precision (such as INT16/INT8) arithmetic. One important reason for this is because the accuracy of a quantized network is hard to retain as the quantization bit-width narrows. 
Yet low-precision accelerators lead to orders of magnitude decrease in chip-area and energy consumption compared to the high-precision hardware \cite{DBLP:journals/corr/abs-1807-03010}. This motivates plenty of researchers to study how to improve the accuracy of quantized networks on low-precision accelerators \cite{rastegari2016xnor,gong2019differentiable,mckinstry2019discovering,jung2019learning,esser2019learned}. While these methods focus on designing low-precision quantization algorithms, the accuracy of quantized networks may be limited by the low computation precision of accelerators. 

Different from the above methods, we try to answer the question: 
\textit{Is it possible to deploy high-precision networks on the existing well-designed low-precision accelerator to capture both model accuracy and inference efficiency?} To achieve this, a naive solution is decomposing a high-precision network at the bit level and conducting inference on low-precision hardware in a nibble iteration manner \cite{DBLP:journals/corr/abs-2101-11748}. For example, to run an 8-bit network on a 2-bit accelerator, we can split each 8-bit operand into four 2-bit operands so that the original 8-bit multiplication can be carried out in $4\times 4=16$ steps using 2-bit multipliers with proper shifting. Although the 2-bit operation is much cheaper than the 8-bit counterpart in terms of chip area and power consumption, there is no gain in inference latency since the total amount of bit operations is unchanged.

In this paper, we investigate the opportunity of fast and accurate inference of high-precision CNN models on low-precision hardware from the algorithm side. We propose PArallel Low-precision Quantization (PalQuant), a hardware-friendly, simple yet effective quantization method for efficient CNN acceleration. Different from the naive bit-level decomposition solution, PalQuant reduces the computational complexity by dividing the expanded low-precision channels into parallel groups.

To encourage information flowing across different groups, we propose a novel cyclic shuffle module that fuses features from two consecutive groups cyclically in a hardware-friendly way. One important property of cyclic shuffle is that it may serve as a complement to channel shuffle. This is mainly due to that cyclic shuffle fuses channel features at group level while the channel shuffle fuses a fraction of channel features from each group.
Extensive experiments from both algorithm and hardware sides demonstrate that PalQuant can consistently achieve the highest accuracy and inference speedup than state-of-the-art methods. 
The contribution of this paper can be summarized as follows:
\begin{itemize}
    \item We propose the PalQuant algorithm that enables efficient high-precision computation in low-precision accelerators via learning parallel low-precision representations from scratch.
    \item We propose a novel cyclic shuffle module to help the information flow across parallel low-precision convolution groups. 
    \item Extensive experiments on  ImageNet benchmark demonstrate that PalQuant outperforms state-of-the-art quantization methods in terms of both accuracy and computational cost. We also examine the speed-up of PalQuant on two CNN accelerators \cite{8416871,Ghodrati2020BitParallelVC}, which shows that PalQuant achieves 1.7$\times$ speed-ups than state-of-the-arts on ResNet-18.
\end{itemize}
\section{Related Work}
\label{section_related_work}
\subsection{Quantization Methods}
Deep Neural Network Quantization methods have become popular in recent years as they can reduce energy consumption and inference latency of deep neural networks. One line of quantization methods aims to reduce the quantization bit-width, while maintaining the network accuracy. 
Early quantization methods mainly focus on  high bit-width fixed-point quantization \cite{gupta2015deep,qiu2016going}, e.g. 8bit or 16bit quantization scheme, which brings in little accuracy degradation. 
Later, various binary \cite{courbariaux2015binaryconnect,rastegari2016xnor,hou2016loss} and ternary quantization methods \cite{lin2015neural,li2016ternary,zhu2016trained} are proposed to reduce the  multiplication operations in the network inference.
Another line of research on quantization methods is to learn good quantization parameters such as quantization step-values, clipping values, and so on. 
Early quantization methods use fixed quantization parameters \cite{hubara2016binarized,gupta2015deep}, dynamic parameters based on statistics of the data distribution \cite{cai2017deep,mckinstry2018discovering}, or seek the parameters that minimize the quantization error \cite{rastegari2016xnor,zhang2018lq}.
Recently, researchers propose to use trainable quantization parameters, e.g. clipping values \cite{lee2021network,choi2018pact} and step-size \cite{esser2019learned,jung2019learning}. These parameters can be learned by gradient back-propagation which minimizes the task loss.
Another related work is WRPN \cite{mishra2017wrpn} which increases the number of filter maps and reduces the quantization bit-width of feature maps. 
Our proposed method differs with WRPN \cite{mishra2017wrpn} in three aspects. First, Our PalQuant and WRPN \cite{mishra2017wrpn}  target at different problems. While  WRPN \cite{mishra2017wrpn} mainly tries to reduce the large memory footprint of high-precision feature maps via  wide reduced-precision representations, our PalQuant enables efficient high-precision computations on low-precision accelerators. Second, PalQuant reduces computational complexity and memory access via parallel low-precision computation scheme, and it achieves higher or comparable network accuracy with less computational cost than WRPN \cite{mishra2017wrpn}. Third, we propose a novel cyclic shuffle module to fuse features across different groups, that further strengthens the model representation. 

\subsection{Hardware Accelerators for CNN}
The extremely high computational complexity of CNN poses a significant challenge to real-world deployment, especially for resource-constraint embedded devices. As a result, energy-efficient FPGA/ASIC-based CNN accelerators have gained increasing popularity recently in both academia and industry. To maintain network accuracy, the computing architecture of early CNN accelerators mainly exploits floating-point and high-precision fixed-point data types (such as 16-bit/8-bit). For example, Zhang \etal \cite{zhang2015optimizing} designs the first FPGA-based accelerator for floating-point CNN inference. DianNao \cite{chen2014diannao}, Eyeriss \cite{chen2016eyeriss}, and TPU \cite{jouppi2017datacenter} are ASIC-based accelerators designed for 16-bit quantized CNN.

With the rapid development of quantization methods in the deep learning community, many low-precision CNN accelerators have been proposed to further improve hardware efficiency. With the help of power-of-two quantization \cite{zhou2017incremental}, Li \etal \cite{li2019system} and Tann \etal \cite{tann2017hardware} propose multiplier-free architectures for area and power-efficient inference by replacing conventional integer multiplications with shift-based operations. YodaNN \cite{andri2017yodann} is an ASIC-based accelerator tailored for Binary Weight Networks \cite{courbariaux2015binaryconnect}. Bit Fusion \cite{8416871}, Bitblade \cite{ryu2019bitblade}, and BPVeC \cite{Ghodrati2020BitParallelVC} are precision-scalable CNN accelerators exploiting 2-bit multipliers for arbitrary-precision computation. Umuroglu \etal \cite{umuroglu2017finn} and Zhao \etal \cite{zhao2017accelerating} propose dedicated accelerators for Binarized Neural Networks \cite{DBLP:conf/nips/HubaraCSEB16}, which can achieve the highest frame-per-second under extremely low area and energy consumption. Lascorz \etal \cite{delmas2019bit} accelerates CNN inference through hardware/software co-design, while PalQuant is more general and can be deployed on a variety of accelerators, including bit-parallel and bit-serial accelerators.
\section{Preliminaries}
\label{section_preliminaries}
\subsection{Notation}
The input activation and weight of a fully-connected layer in a deep neural network are denoted by $X\in \mathcal{R}^{N\times S}$ and $W\in \mathcal{R}^{T\times S}$, respectively. The layer's output $Y$ can be obtained by:
\begin{align}\label{fc}
Y=XW^T
\end{align}
The quantized input activation and weight are denoted by  $\hat{X}\in \mathcal{R}^{N\times S}$ and $\hat{W}\in \mathcal{R}^{T\times S}$, respectively. For a convolutional layer, the output $Y$ can also be obtained by Eq. \ref{fc}, since the weights and activations can be transformed into matrices. 
\subsection{Uniform Quantization}
% Uniform quantization maps floating-point numbers to low-precision fixed-point representations via equally spaced quantization levels.
Recently uniform quantizer with trainable quantization parameters \cite{gong2019differentiable,lee2021network} becomes popular. Given a floating-point number $x$, it first maps $x$ to a range $\left[0,1\right]$ by the following equation:
\begin{align}
    {x}_n=clip\left(\frac{x-l}{u-l},0,1\right)
\end{align}
where $clip(\cdot,\cdot,\cdot)$ denotes the clip function, $l$ and $u$ denotes lower bound and up bound, respectively. Then the integer value $x_q$ can be obtained by the quantization function:$x_q=\lfloor (2^b-1)x_n\rceil$
, where $\lfloor\cdot\rceil$ is rounding-to-nearest operation, and $b$ denotes the quantization bit-width. And the de-quantization function is given by the following formula:
\begin{align}
    \hat{x}=\begin{cases}
    \frac{x_q}{2^b-1}& x \in weight\\
    2(\frac{x_q}{2^b-1}-0.5)& x \in activation
    \end{cases}
\end{align}
where the formula maps the de-quantized weights to a range $\left[-1,1\right]$, and restricts the de-quantized activations to be non-negative. An additional parameter $\alpha$ \cite{lee2021network}  is required to  multiply the output activations of each layer, which adjusts the  output scale of the whole feature map. 
\subsection{Bit-level Decomposing}
Consider that there is an accelerator that supports only $B$-bit computation, a naive solution for running a high-precision ($M$-bit) model on this accelerator is to decompose $M$-bit representations to multiple $B$-bit  at bit-level, then shift and sum up those $B$-bit operations results: 
 \begin{align}\label{bit_decompose}
    Y=\hat{X}\hat{W}^T=\sum\limits_{i=0}^{G-1}\sum\limits_{j=0}^{G-1}\hat{X}_i\hat{W}_j^{T}*2^{(i+j)*B}
\end{align}
where $\hat{X}_i$ and $\hat{W}_j$ are $B$-bit input and weights that are generated by splitting $M$-bit input and weights in bit-level, and $G=ceil(M/B)$. Note that the computational cost is unchanged comparing to the original $M$-bit computation.
\begin{figure}
  \subfigure[]{\includegraphics[width=0.48\linewidth]{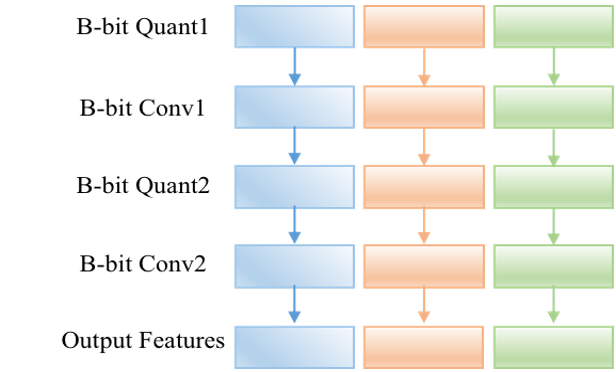}
  \label{parallel_block_fig} }
  \subfigure[]{\includegraphics[width=0.48\linewidth]{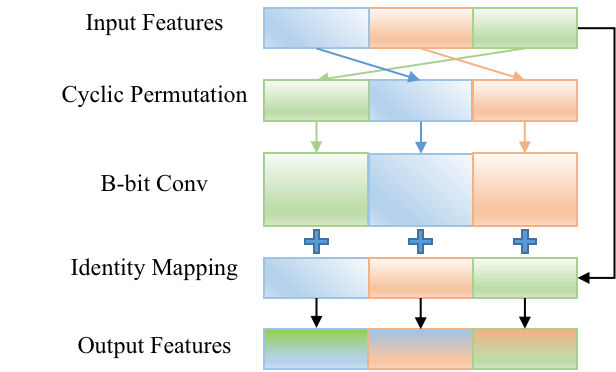}
  \label{cyclic_shuffle_fig} }
  \caption{(a) The proposed parallel low-precision computation scheme. Here the hyper-parameter $G$ is fixed to 3. B-bit Quant and B-bit Conv denote a  B-bit quantization layer and B-bit convolutional layer, respectively. (b) The proposed cyclic shuffle module. Here B-bit Conv denotes a B-bit quantized 1$\times$1 group convolutional layer with $G$=3}
\end{figure}
\section{Proposed Method}
\label{section_proposed_method}
\subsection{Parallel Low-Precision Quantization}
Since the bit-level decomposing method (Eq.\ref{bit_decompose}) has the same amount of bit operations as the original high-precision computation, there is no gain in inference latency. At first, we try to reduce the inference latency by an approximation to Eq.\ref{bit_decompose}:
 \begin{align}\label{outpout}
    min_{\bar{W_i}}&=\lVert\sum\limits_{i=0}^{G-1}\sum\limits_{j=0}^{G-1}\hat{X}_i\hat{W}_j^{T}*2^{(i+j)*B}-\sum\limits_{i=0}^{G-1}(\hat{X_i}\bar{W_i}^{T})*2^{i*B}\rVert_F^2
\end{align}
where $\hat{W_j}$ and $\hat{X_i}$ are $B$-bit representations chunked from $M$-bit $\hat{W}$ and $\hat{X}$, respectively. We found that  the $B$-bit computations in Eq.\ref{outpout} are naturally in parallel $G$ groups, which provides the potential to be hardware-friendly. In addition, Eq.\ref{outpout} approximates the standard results with only $G$ parallel $B$-bit computations while  Eq.\ref{bit_decompose} requires $G^2$ $B$-bit computations. This indicates that the proposed approximation only costs $\frac{1}{G}$ computation complexity of the bit-level decomposition method. 

Yet there are two problems in Eq.\ref{outpout}. One is that we empirically found that all $G$ solutions of $\bar{W_i}$ degenerate to one similar solution, which results in an inferior network accuracy. We assume that this phenomenon has relations with two aspects: 1. $B$-bit representations are chunked from $M$-bit ones, thus they are coupling together with a strong connection\footnote{For example, given the 4-bit number $x$, the lowest 2-bit number changes from 3 to 0 when $x$ changes from 3 to 4 ($\left[0\ 0\ 1 \ \ 1\right]\rightarrow \left[0\ 1\ 0 \ 0\right]$)}; 2. minimizing the feature map reconstruction loss easily causes the over-fitting problem.
The other problem in Eq.\ref{outpout} is that chunking $M$-bit $\hat{X}$ to multiple $B$-bit representations tend to be inferior to learning $G\times$ $B$-bit representations directly by back-propagation, as mentioned in \cite{mishra2017wrpn}.

To tackle the above problems, we propose the parallel low-precision quantization scheme  by learning multiple low-precision representations in parallel. It inherits the computation efficiency of  parallel low-precision computation from Eq.\ref{outpout}, and cures the solution degeneration problem via training $B$-bit representations through back-propagation.
Specifically, we expand the activation channels of a convolutional layer by $G\times$, and split these activations into $G$ groups along the channel dimension. Then we quantize the activations and weights to $B$-bit in each group, and there are $G$ groups of $B$-bit weights in total. 
% Since we directly quantize the expanded activations to $B$-bit, there is no need for shifting the output results. 
Fig.\ref{parallel_block_fig} depicts the proposed parallel low-precision quantization scheme, which can be  implemented naturally by quantized group convolution with expanded channels.
\subsection{Cyclic Shuffle}
% \subsection{Information Exchange by Channel Shuffle}
Although the proposed parallel low-precision computation scheme enjoys wonderful hardware efficiency, it hinders information flowing across different groups of feature channels. 
To cure this problem, we propose a novel cyclic shuffle module to enhance information communication across different groups. 
First, we adopt cyclic permutation \cite{bogart1989introductory} to permute the channels in group-level. Given input $X$ with shape $\left[N\times C \times H \times W\right]$, it can be reshaped to $\left[N\times G \times \hat{C}\times  H \times W\right]$ where $\hat{C}=C/G$.  We propose to re-order the channels  in group-level via the cyclic permutation. Specifically, let $\mathbf{S}=\{ 0,1,2,3\cdots G-1\}$ denotes group indexing set, a cyclic permutation function  $\pi(\cdot)$ for $\mathbf{S}$ is 
\begin{equation}\label{permute_function}
    \pi(i)\ = \quad (i+1)\ mod\ G
\end{equation}
A more clear notation of $\pi(\cdot)$ by Cauchy's two-line notation \cite{wussing2007genesis} is:
\begin{equation}
\begingroup % keep the change local
\setlength\arraycolsep{8pt}
    \begin{pmatrix}
    0&1 & 2 & 3 & 4 & \cdots &G-2& G-1\\
    1&2 & 3 & 4 & 5 &\cdots &G-1& 0\\
    \end{pmatrix} 
    \endgroup
\end{equation}
where the first line denotes the channel group indexing set $\mathbf{S}$ and the second line denotes the corresponding order after permutation. After  cyclic permutation, the output $Y$ can be represented by:
\begin{equation}\label{permute}
    {Y}[:,\pi(i),:,:,:]={X}[:,i,:,:,:]
\end{equation}
By using Eq.\ref{permute}, we can re-order feature channels  by permuting the group-level feature cyclically. Under this scheme, the information from each group is received by its following one, which helps the information exchange between feature channels of different groups. 
Although we have exchanged the information between different groups by the cyclic permutation, each group still holds only one group of feature channels. In order to fuse the feature channels from different groups, we build a novel module, named cyclic shuffle, that composes of  the cyclic permutation and a 1$\times$1 group convolution with short-cut connection. Given the input $X$, the output $Z$ of the cyclic shuffle module can be represented by the following formulas:
\begin{align}\centering
    &{Y}=\mathrm{CyclicPermute}({X})\\
    &{Z}={X}+\mathrm{Conv}_{1\times 1}({Y})
\end{align}
where $\mathrm{CyclicPermute}(\cdot)$ denotes the cyclic permutation function (refer to Eq.\ref{permute}).
In the cyclic shuffle module, the cyclic permutation re-orders the  feature channels at group level, then the 1x1 group convolution learns a mapping function for channel fusing, finally the short-cut connection together with the element-wise addition operation aggregate feature channels  from different groups. Fig. \ref{cyclic_shuffle_fig} depicts the proposed cyclic shuffle module. 

Our proposed cyclic shuffle module enjoys benefits from three aspects: 1. According to Proposition. \ref{propss}, after passing through G-1 cyclic shuffle modules, each group receives the information flows from all groups of feature channels, that helps cure the side-effect of group convolution. 2. The proposed cyclic shuffle module fuses feature channels at group level, which makes it  possible for the intermediate results  to stay on
chip. As a result, it saves data access to the out-chip memory significantly. 3. Our cyclic shuffle module fuses feature channels in group-level  while channel shuffle aggregates a part of feature channels from each group. Therefore, cyclic shuffle can be combined with channel shuffle to achieve higher network accuracy.
\newtheorem{prop}{Proposition}
\begin{prop}\label{propss}
Suppose the number of groups is ${G}$, then ${G}$-${1}$ cyclic shuffle modules are enough to assure that each group contains the information flows from all parallel groups.
\end{prop}
\begin{proof}
The proof is simple. The cyclic permutation $\pi(\cdot)$  in Eq.\ref{permute_function} is a $G$-cycle \cite{bogart1989introductory}, which means the group indexes will be the same as the original order after ${G}$ times cyclic permutation. As each group fuses a the information flow from its previous group after passing through one cyclic shuffle module, then ${G}$-1 cyclic shuffle modules are enough to assure that each group contains all information from $G$ group channels. 
\end{proof}
\begin{figure}[htbp]
\centering
\includegraphics[width=0.9\linewidth,scale=1.0]{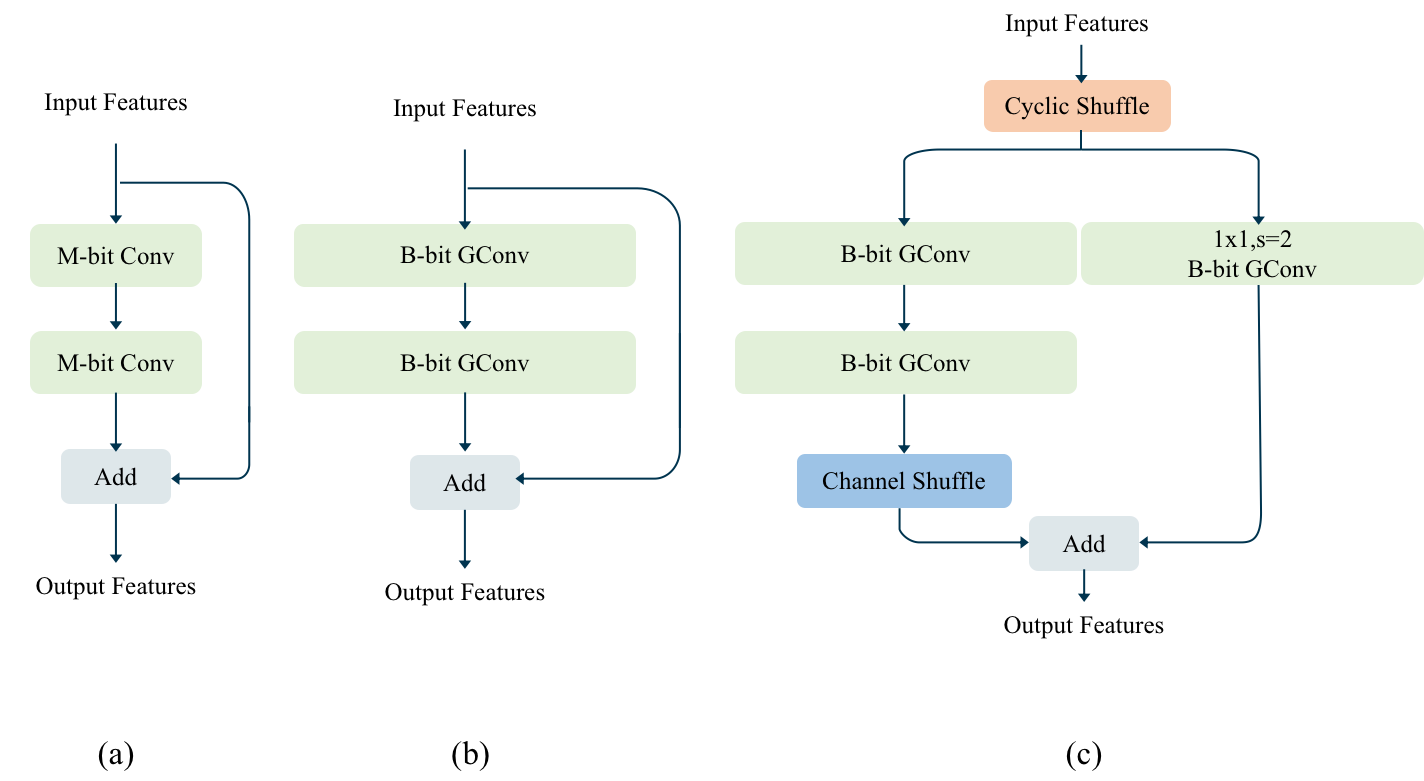}
\caption{(a) A typical building block of $M$-bit quantized ResNet network; (b) A typical building block of PalQuant with $G$=2;(c) A typical building block of PalQuant with shuffle modules (stride=2, $G$=2)}
\label{overall_blocks}
\end{figure}
% \vspace{-1.0cm}
\subsection{Overall Framework}
In this subsection, we give the overall framework of our proposed method. 
To approximate $M$-bit computations on $B$-bit accelerator, we propose the parallel low-precision quantization scheme which expands the channel dimensions of each layer by $G$ times ($G=ceil(M/B)$) and splits them into $G$ parallel groups. Taking ResNet network quantization for an example, PalQuant makes a simple modification to the standard $M$-bit basic building block (Fig. \ref{overall_blocks} (a)): enlarging the input channels by $G\times$ and using $B$-bit quantized group convolution with the number of groups equals to $G$. The basic building block of PalQuant is depicted in Fig. \ref{overall_blocks} (b).

To cure the side-effect caused by the group convolution, we propose the cyclic shuffle module to help information communication between different groups. Considering that the cyclic shuffle module contains a $1\times 1$ group convolution layer, we only place the cyclic shuffle module at the beginning of each stage
% \footnote{Here a stage in CNNs denotes those layers that share the same feature map resolution.}
, which brings little extra computational cost.  
Since cyclic shuffle and channel shuffle fuse information flows in different ways, we use channel shuffle module as a complement. As the channel shuffle in the middle of two consecutive group convolutions will impair the parallel computation flows, we move the channel shuffle module  to the end of the convolution block.
 A typical building block of PalQuant with shuffle modules is depicted in Fig.\ref{overall_blocks} (c).
\section{Experiments}
\label{section_experiments}
In this section, we first give details of experiment settings, then we compare our proposed algorithm PalQuant with state-of-the-art quantization algorithms. To demonstrate the hardware efficacy of PalQuant, we further conduct extensive experiments on CNN accelerators. 
Finally, we show the ablation study results and the generalization of PalQuant under different hyper-parameter $B$ and $G$.\\
\textbf{Network Architectures and Datasets.} Since ResNet \cite{he2015deep} is widely used in various computer vision tasks, here we adopt ResNet-18 and ResNet-34 as benchmark networks. 
% In addition, we also use the VGG16-BN network in our experiments to verify the proposed method on deep networks without short-cut connections.
All the experiments in this work are conducted on ImageNet dataset which has over 1 million training images and 50,000 validation images.\\
\textbf{Training Details.} In this work, we use uniform quantization for the input activation and weight of each convolutional or fully-connected layer. In the training stage, a straight-through estimator(STE) \cite{bengio2013estimating} is used to approximate the gradient through the rounding function. Based on this gradient estimator, we learn the lower bound $l$, upper bound $u$, and the scaling factor $\alpha$ through gradient back-propagation. 
Following experiment settings in previous work \cite{lee2021network,esser2019learned}, the first and the last layer are not quantized unless otherwise specified.
% the intermediate layers' feature map channels are expanded by $G \times$, the input channels of the first layer and the output channels of the last layer are unchanged.
We implement our method under the Pytorch framework and train all the networks on a GPU server with eight Nvidia Titan 3090 GPUs. 
We adopt an SGD optimizer with the momentum set to 0.9. The weight decay is set to 1e-4, and the learning rate is initialized to 0.01 and 
adjusts with a cosine learning rate decay strategy.
Following previous work \cite{esser2019learned}, we train all the models from scratch for 90 epochs with a batch size of 256.
In the training stage, random cropping, resizing, and horizontal flipping are adopted while only resizing and center cropping are used in the validation stage.
% \begin{table}[htbp]
\begin{table}[t]
\caption{\textbf{Comparison Results of Quantized ResNet-18 and ResNet-34 on ImageNet.} $\dagger$ Results of WRPN \cite{mishra2017wrpn} on ResNet34 are taken from the paper}
\begin{center}
\tablestyle{1.0pt}{1.2}
\begin{tabular}{c|c|c|c|c|c|c}
\shline
% \textbf{Methods}&\textbf{BitWdith}&\textbf{OPS}&\textbf{Top1 Acc.}\\
\multirow{2}{*}{Method} &\multirow{2}{*}{Strategy} & \multirow{2}{*}{Precision} & \multicolumn{2}{c|}{ResNet-18} & \multicolumn{2}{c}{ResNet-34}\\
\cmidrule{4-7}
& &&BitOps & Top1 Acc.& BitOps & Top1 Acc.  \\
\hline
DSQ  \cite{gong2019differentiable}  & Fine-tuned& 4b A,4b W &29.10G & 69.60 & 58.74G&72.80 \\ 
FAQ \cite{mckinstry2019discovering} & Fine-tuned&4b A,4b W &29.10G & 69.80 &58.74G &73.30\\ 
QIL \cite{jung2019learning}  & Fine-tuned&4b A,4b W &29.10G & 70.10 & 58.74G& 73.70 \\ 
EWGS \cite{lee2021network} & Fine-tuned&4b A,4b W &29.10G & 70.60 &58.74G & 73.90 \\ 
LSQ \cite{esser2019learned} & Fine-tuned &4b A,4b W &29.10G & 70.70 &58.74G & 73.50\\ 
% LSQ+\cite{bhalgat2020lsq+}& 66.8 & 70.8 & - & - \\ \hline
PalQuant(\textbf{Ours}) &Scratch &\textbf{B=2,G=2} &\textbf{14.87G }& \textbf{71.12} & \textbf{29.68G}& \textbf{73.90}\\
PalQuant(\textbf{Ours}) &Scratch&\textbf{B=2,G=3} &\textbf{22.30G} & \textbf{72.36} & \textbf{44.53G}&\textbf{74.52}\\ 
PalQuant(\textbf{Ours}) &Scratch&\textbf{B=2,G=4} &29.73G & \textbf{72.74} & 59.37G&\textbf{74.96}\\ \hline
% DSQ  \cite{gong2019differentiable}  & 6-bit & & -  \\ 
% FAQ \cite{mckinstry2019discovering} &6-bit & & - \\ 
% QIL\cite{jung2019learning}  &6-bit & & -  \\ 
EWGS \cite{lee2021network} & Fine-tuned&6b A,6b W &65.48G & 71.01 &132.16G & 74.14 \\ 
LSQ \cite{esser2019learned}  & Fine-tuned&6b A,6b W &65.48G & 71.59 &132.16G &74.43 \\ 
% LSQ+\cite{bhalgat2020lsq+}& 66.8 & 70.8 & - & - \\ \hline
WRPN \cite{mishra2017wrpn}&Scratch&2$\times$,2b A,2b W  & 29.10G& 72.31 & 58.74G&73.32$\dagger$\\ 
PalQuant(\textbf{Ours}) &Scratch&\textbf{B=2,G=3} &\textbf{22.30G} & \textbf{72.36} & \textbf{44.53G}&\textbf{74.52}\\ 
\hline
% DSQ  \cite{gong2019differentiable}  & 8-bit & & -  \\ 
FAQ \cite{mckinstry2019discovering} & Fine-tuned&8b A,8b W &116.4G & 70.00 &234.94G &73.70 \\ 
% QIL\cite{jung2019learning}  &8-bit & & -  \\ 
LSQ \cite{esser2019learned}  &Fine-tuned&8b A,8b W & 116.4G& 71.10  &234.94G & 73.98\\ 
% LSQ+\cite{bhalgat2020lsq+}& 66.8 & 70.8 & - & - \\ \hline
EWGS \cite{lee2021network} & Fine-tuned&8b A,8b W &116.4G & 71.14 &234.94G & 73.97\\ 
WRPN \cite{mishra2017wrpn}&Scratch&2$\times$,2b A,2b W &29.10G & 72.31 & 58.74G&73.32$\dagger$\\ 
PalQuant(\textbf{Ours}) &Scratch&\textbf{B=2,G=4} &\textbf{29.73G} & \textbf{72.74} &\textbf{59.37G}&\textbf{74.96}\\ 
% \hline
% -  & FullPrecision &1.819G & 70.50 & 1.819G& 74.10 \\ 
\shline
% copy& More table copy$^{\mathrm{a}}$& &  \\
% \hline
% \multicolumn{4}{l}{$^{\mathrm{a}}$Sample of a Table footnote.}
\end{tabular}
\label{resnet18_compare}
\end{center}
\end{table}
\subsection{Comparison with State-of-the-arts Algorithms}
% \subsection{Comparison with Other Quantization Methods}
We first compare our proposed PalQuant with state-of-the-art quantization algorithms. All the accuracy results of other methods, except for LSQ \cite{esser2019learned} and WRPN \cite{mishra2017wrpn}, are taken from corresponding papers. As LSQ \cite{esser2019learned} reports the results of pre-activation ResNet, we re-implement LSQ \cite{esser2019learned} under Pytorch framework and get the quantization results on standard ResNet. And we re-implement  WRPN \cite{mishra2017wrpn} via the same quantization function as ours, and get better results than the reported ones in \cite{mishra2017wrpn}. 

Table \ref{resnet18_compare} shows the top-1 accuracy of quantized ResNet-18 and ResNet-34 with different quantization precision. From the table, we can observe that PalQuant outperforms state-of-the-art methods in both accuracy and computational complexity. In detail, PalQuant  ($G$=3) achieves +1.66$\%$ (72.66$\%$ v.s. 70.7$\%$)  higher top-1 accuracy than  4-bit LSQ \cite{esser2019learned} with $23.4\%$ (22.3G v.s. 29.1G) less computational cost on ResNet-18. With comparable computational budget, PalQuant ($G$=4) outperforms 4-bit EWGS \cite{lee2021network} by 1.06$\%$ on ResNet-34. Note that both LSQ\cite{esser2019learned} and EWGS\cite{lee2021network} are finetuned from the pre-trained model, while our method are trained from scratch.
In addition,  PalQuant($G$=3) obtains slightly higher top-1 accuracy  than WRPN \cite{mishra2017wrpn} with $23.4\%$ less bitOps on ResNet-18. These results suggest that the proposed PalQuant has strong feature representation power, which leads to the highest accuracy in the experiments under less or equal computational budgets.
\begin{figure}[htbp]
\centering
\includegraphics[width=0.75\linewidth,scale=1.0]{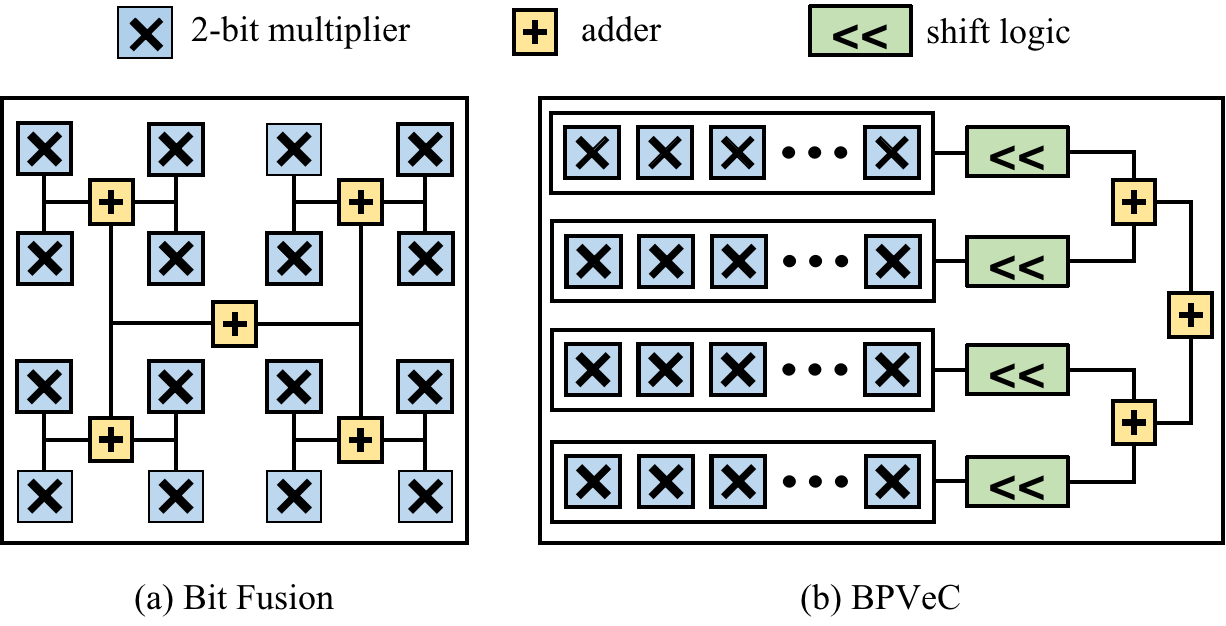}
\caption{The basic building blocks for 2-bit inference in Bit Fusion and BPVeC}
\label{dla_fig}
\end{figure}
\subsection{Efficacy on Hardware Acceleration}
We select the state-of-the-art CNN accelerator Bit Fusion \cite{8416871} and BPVeC \cite{Ghodrati2020BitParallelVC} to demonstrate the efficacy of the proposed PalQuant. Bit Fusion \cite{8416871} is a systolic array-based accelerator designed for low-precision inference, it employs a bit-level decomposable fusion unit for 2/4/8-bit multiplication. To further improve hardware efficiency, BPVeC \cite{Ghodrati2020BitParallelVC} exploits the Narrow Bit-width Vector Engines (NBVE), an energy and area efficient alternative to conventional MAC that consists of multiple 2-bit multipliers and an adder tree for SIMD-based inner-product computation. The basic building blocks of these two accelerators are shown in Fig. \ref{dla_fig}. In this experiment, we compare our proposed method against LSQ \cite{esser2019learned} in terms of inference latency and energy consumption. To mimic resource-constraint computing platforms, we configure both accelerators to 2-bit inference mode with a small on-chip buffer. In this scenario, the on-chip buffer cannot accommodate all weights and activations of a layer, leading to non-negligible data traffic in the memory hierarchy. We develop a cycle-accurate simulator to collect statistics of computation and on-chip buffer access. The power of computational logic and SRAM-based on-chip buffer under 45nm technology are directly drawn from \cite{8416871} and \cite{Ghodrati2020BitParallelVC}. For off-chip data traffic, we assume 15pJ/bit per DDR4 access as in \cite{Ghodrati2020BitParallelVC}. 

Table. \ref{hardware_exp} shows the results of hardware acceleration. It is clear to see that our proposed method consistently outperforms the baseline in performance, energy efficiency, and accuracy for the same precision of weight and activation. For example, PalQuant ($G$=2)  obtains 0.52\% higher accuracy, 1.78$\times$ speedup, and 1.91$\times$ energy efficiency simultaneously over 4-bit LSQ \cite{esser2019learned} on Bit Fusion accelerator \cite{8416871}. The benefit mainly stems from the channel grouping in our method. On the one hand, channel grouping leads to reduced bitOps, which is the source of performance gain. On the other hand, channel grouping eliminates the data dependency between different groups, resulting in reduced on-chip buffer size and off-chip bandwidth requirement, which contributes to the low energy consumption. Besides, the channel shuffle unit at the end, rather than the middle, of a block encourages the intermediate results of group convolution staying on chip, which is another key for saving energy.

In the paper, we evaluate PalQuant on Bit Fusion and BPVeC, yet our algorithm can also witness similar benefits on other bit-parallel accelerators \cite{ryu2019bitblade,camus2019review} which implement a high-precision multiplication as parallel low-precision multiplications followed by reduction. The fundamental reason is that compared to conventional low-precision quantization, our method can always reduce the number of computations and buffer accesses. 
\begin{table}[htbp]
\caption{\textbf{Comparison of inference latency and energy consumption on two hardware accelerators.} We use hyper-parameter $B$=2}
\begin{center}
\tablestyle{0.4pt}{1.2}
\begin{tabular}{c|c|c|c|c|c|c|c|c}
\shline
\multirow{3}{*}{DLA} &\multirow{3}{*}{Method} & \multirow{3}{*}{Precision} & \multicolumn{3}{c|}{ResNet-18} & \multicolumn{3}{c}{ResNet-34}\\
\cmidrule{4-9}
& &&\multirow{2}{*}{SpeedUp}& Energy& Top1& \multirow{2}{*}{SpeedUp}& Energy& Top1  \\
& &&& Efficiency& Acc.& & Efficiency&  Acc.\\
\hline
& LSQ \cite{esser2019learned}& 4b A,4b W&1x &1x& 70.70 & 1x&1x&73.50 \\ 
 & \textbf{Ours}&\textbf{G=2} &\textbf{1.78x} &\textbf{ 1.91x} &\textbf{71.12} &\textbf{1.78x}&\textbf{1.91x}&\textbf{73.90} \\ 
\cline{2-9}
Bit & LSQ \cite{esser2019learned}&6b A,6b W &1x & 1x &71.59 &1x&1x&74.43\\ 
Fusion \cite{8416871} & \textbf{Ours} &\textbf{G=3}&\textbf{2.52x} &\textbf{ 2.78x} &\textbf{72.36} &\textbf{2.50x}&\textbf{2.78x}&\textbf{74.52}\\ 
\cline{2-9}
 & LSQ \cite{esser2019learned}&8b A,8b W &1x & 1x &71.14 &1x&1x&73.97\\ 
 & \textbf{Ours} &\textbf{G=4}&\textbf{3.21x} & \textbf{3.60x} &\textbf{72.71} &\textbf{3.13x}&\textbf{3.60x}&\textbf{74.96}\\ 
\hline
\multirow{6}{*}{BPVeC  \cite{Ghodrati2020BitParallelVC}}  & EWGS \cite{lee2021network}& 4b A,4b W&1x &1x&70.60&1x&1x&73.90 \\ 
& \textbf{Ours}&\textbf{G=2} &\textbf{1.77x} & \textbf{1.92x} &\textbf{71.12} &\textbf{1.87x}&\textbf{1.74x}&\textbf{73.90} \\ 
\cline{2-9}
 & EWGS \cite{lee2021network}&6b A,6b W &1x & 1x &71.01 &1x&1x&74.14\\ 
 & \textbf{Ours} &\textbf{G=3}&\textbf{2.56x} &\textbf{ 2.84x} &\textbf{72.36} &\textbf{ 2.46x}&\textbf{2.81x}&\textbf{74.52}\\ 
\cline{2-9}
 & EWGS \cite{lee2021network}&8b A,8b W &1x & 1x &71.14 &1x&1x&73.97\\ 
 & \textbf{Ours} &\textbf{G=4}&\textbf{3.12x} & \textbf{3.70x} &\textbf{72.71} &\textbf{3.07x}&\textbf{3.67x}&\textbf{74.96}\\ 
\shline
% copy& More table copy$^{\mathrm{a}}$& &  \\
% \hline
% \multicolumn{4}{l}{$^{\mathrm{a}}$Sample of a Table footnote.}
\end{tabular}
\label{hardware_exp}
\end{center}
\end{table}
\subsection{Ablation Study}
In this subsection, we ablate the designs of the PalQuant algorithm based on the ResNet-18 network with hyper-parameter $B$=2 and $G$=2, and all the experiments follow the same training settings as above.
\paragraph{Ablation: Influence of Different Shuffle Modules}
Here we explore the benefits of the proposed cyclic shuffle module.
Table. \ref{per-block} provides the comparison of different shuffle modules. The baseline model in this table doesn't contain either cyclic shuffle or channel shuffle. As we can see, using  cyclic shuffle or  channel shuffle separately achieves similar accuracy improvements. This indicates that both modules help the information communication between different groups. Besides, using cyclic shuffle and channel shuffle jointly demonstrate the highest accuracy gain (up to 3.19$\%$) in terms of top-1 accuracy. This phenomenon verifies our assumption that cyclic shuffle and channel shuffle fuse information flow in different granularities and they may serve as a complement for each other. 
    
We also explore the way to use channel shuffle. Table \ref{per-block} shows that using per-stage or per-block channel shuffle has similar performance without cyclic shuffle. Together with cyclic shuffle, per-stage channel shuffle has higher accuracy improvement(+3.19$\%$ Top-1 accuracy on ImageNet) than per-block channel shuffle. In this paper, we use per-stage channel shuffle by default unless otherwise specified.
\begin{table}[htbp]
\caption{\textbf{Influence of different shuffle modules}}
\begin{center}
\label{per-block}
\tablestyle{12pt}{1.2}
\begin{tabular}{c|c|c|c}
\shline
Cyclic Shuffle & Channel Shuffle & Top1 Acc.& Top5 Acc.\\
\hline
&  & 67.94 & 87.90\\
& per-block& 70.31 &89.58\\
& per-stage& 70.24 &89.51\\
\checkmark&  & 70.22 & 89.55\\
\checkmark & per-block& 70.62 &89.69\\
\checkmark & \textbf{per-stage} & \textbf{71.13} & \textbf{89.99}\\
\hline
\multicolumn{2}{c|}{$\Delta$} & \textbf{3.19}$\uparrow$ & \textbf{2.09}$\uparrow$\\
\shline
\end{tabular}
% \label{tab1}
\end{center}
\end{table}
\paragraph{Ablation: Importance of Cyclic Permutation.}
As the cyclic shuffle module brings in one extra 1$\times$1 group convolution, one may concern that most of the accuracy gains come from the additional computational cost. Experimental results in Table. \ref{CyclicPermute} are against this opinion. PalQuant loses 0.76 $\%$ top-1 accuracy after dropping the cyclic permutation and keeping other parts unchanged while dropping the whole cyclic shuffle module causes 0.9 $\%$ top-1 accuracy decline. This indicates that the gain brought by the cyclic shuffle module is mainly due to cyclic permutation.
\begin{table}[htbp]
\caption{\textbf{Importance of cyclic permutation}}
\begin{center}\label{CyclicPermute}
\tablestyle{10pt}{1.2}
\begin{tabular}{c|ll|ll}
\shline
Module & \multicolumn{2}{l|}{Top1 Acc.}& \multicolumn{2}{l}{Top5 Acc.}\\
\hline
PalQuant& \textbf{71.13}& &\textbf{89.99}&\\
w.o cyclic permutation & 70.37\cgaphlp{-0.76} & &89.57\cgaphlp{-0.42}&\\
w.o cyclic shuffle & 70.23\cgaphlp{-0.90} & &89.51\cgaphlp{-0.48}&\\
\hline
\shline
\end{tabular}
\end{center}
\end{table}
% \vspace{-1.0cm}
\subsection{Generalization}
\paragraph{Results With Different Numbers of Parallel Groups}
In this section, we explore the generalization of PalQuant under different numbers of parallel groups $G$. 
In Table. \ref{base_parallel_groups}, our proposed method outperforms  LSQ \cite{esser2019learned} consistently on a various number of parallel groups, which demonstrates the generalization ability of PalQuant. In addition, PalQuant shows superior performance (nearly 3$\%$ higher top-1 accuracy) than LSQ \cite{esser2019learned} under the same computational budget. 
\begin{table}[htbp]
\caption{\textbf{Results with different numbers of parallel groups}}
\begin{center}\label{base_parallel_groups}
\tablestyle{10pt}{1.0}
\begin{tabular}{c|c|c|c|c}
\shline
Method & Precision & BitOps & Top1 Acc. &Top5 Acc.\\
\hline
% \hline
    LSQ \cite{esser2019learned}  &  4b A,4b W  & 29.10G&70.99 &	89.86  \\ 
    LSQ \cite{esser2019learned}  &  4b A,2b W &14.55G &69.18 &	88.93  \\ 
    PalQuant(\textbf{Ours})  & B=2,G=2   &14.87G & \textbf{71.13} & \textbf{89.99}  \\ \hline
     LSQ \cite{esser2019learned}  &  6b A,6b W &65.48G &71.58 &	90.24 \\ 
    LSQ \cite{esser2019learned} &  6b A,2b W &21.83G & 69.47 &	89.06 	   \\ 
    PalQuant(\textbf{Ours})  & B=2,G=3   &22.30G & \textbf{72.36}&	\textbf{90.63}  \\ \hline
       LSQ \cite{esser2019learned}   &  8b A,8b W  & 116.40G& 71.10 &	90.10  \\ 
    LSQ \cite{esser2019learned}   &  8b A,2b W &29.10G & 70.71 &	89.70	   \\ 
    PalQuant(\textbf{Ours})  & B=2,G=4   &29.73G & \textbf{72.74}&	\textbf{90.90}  \\ \shline
\end{tabular}
% \label{tab1}
\end{center}
\end{table}
\paragraph{Results With Different Quantization Precision}
Next, we explore the performance of PalQuant under different quantization precision $B$. Table. \ref{base_bitwidth} shows that PalQuant outperforms LSQ \cite{esser2019learned} consistently across different low-precision bit-width, i.e. $B$=2, 3, and 4. 
% Besides, the accuracy of our method improves as the low-precision bit-width $B$ increases. 
This result means that our proposed scheme can be applied to CNNs accelerators with different quantization precision. We also quantize the weights with only half of activation bit-width in LSQ \cite{esser2019learned}, which has nearly the same computation complexity as PalQuant. The result shows that our method outperforms LSQ \cite{esser2019learned} by a large margin under the comparable computational budget. The result of LSQ\cite{esser2019learned} with 6/3-bit quantization is provided in Appendix.
% Given fixed G=2, we quantize ResNet-18 to 4-bit, 6-bit, and 8-bit.
% After building a parallel low-precision quantized network architecture, we use the training algorithm proposed in LSQ to train the parallel low-precision quantized ResNet18 network.
\begin{table}[htbp]
% \caption{\textbf{Comparison Results of Quantized ResNet-18 with Different Low-precision Bit-width $B$}}
\caption{\textbf{Results with different quantization precision}}
\begin{center}\label{base_bitwidth}
\tablestyle{10pt}{1.0}
\begin{tabular}{c|c|c|c|c}
\shline
Method&Precision&BitOps&Top1 Acc. &Top5 Acc.\\
\hline
% \hline
 LSQ \cite{esser2019learned} &  4b A,4b W & 29.10G& 70.99 &	89.86  \\ 
LSQ \cite{esser2019learned} &  4b A,2b W & 14.55G& 69.18 &	88.93  \\ 
PalQuant(\textbf{Ours})  & B=2,G=2   &14.87G & \textbf{71.13} & \textbf{89.99}  \\ \hline
LSQ \cite{esser2019learned}  &  6b A,6b W  &65.48G &71.58 &	90.24 \\ 
LSQ \cite{esser2019learned}  &  6b A,3b W & 32.74G& 70.896 &	89.712  \\ 
PalQuant(\textbf{Ours}) & B=3,G=2   & 33.45G&\textbf{72.72}&	\textbf{   90.90}   \\ \hline
LSQ \cite{esser2019learned}  &  8b A,8b W  &116.40G & 71.10 &	90.10  \\ 
LSQ \cite{esser2019learned}  &  8b A,4b W &58.20G & 71.46&	90.15  \\ 
PalQuant(\textbf{Ours})  & B=4,G=2   &59.46G & \textbf{73.48} & \textbf{91.34}  \\ \shline
% copy& More table copy$^{\mathrm{a}}$& &  \\
% \hline
% \multicolumn{4}{l}{$^{\mathrm{a}}$Sample of a Table footnote.}
\end{tabular}
% \label{tab1}
\end{center}
\end{table}
\section{Conclusion}
\label{section_conclusion}
In this paper, we propose the PalQuant, a parallel low-precision quantization method that achieves efficient and accurate high-precision network inference on low-precision accelerators. To help the information flow across parallel groups, we propose the cyclic shuffle module to aggregate parallel information at group level. Extensive experiments demonstrate that PalQuant outperforms state-of-the-art quantization methods in terms of both accuracy and computational complexity.
\paragraph{Acknowledgements}
This work was supported in part by National Key Research and Development Program of China  (Grant No.2021ZD0201504), and National Natural Science Foundation of China (Grant No.62106267).
\appendix
\renewcommand{\thesection}{\Alph{section}}
\section*{Appendix}
\section{Improve ShuffleNet-v2 with Cyclic Shuffle}
As mentioned in the main-body of the paper, the proposed cyclic shuffle  shows its well performance on group convoltuions for quantized ResNet-18 and ResNet-34. Since we assume that it can serve as a complement of channel shuffle, here we explore the performance of cyclic shuffle on ShuffleNet-v2.  

As we didn't find the official training codes of ShuffleNet-v2, we re-implemented the ShuffleNet-v2 training algorithm under Pytorch framework.
We adopt an SGD optimizer with the momentum set to 0.9. The weight decay is set to 1e-4, and the learning rate initialized by 0.1 is adjusted with a cosine learning rate decay strategy. All models in this subsection are trained from scratch for 300 epochs with a batch size of 512. 

As shown in table \ref{shufflenet_v2}, we can achieve +0.48\% higher top-1 accuracy by adding one cyclic shuffle module at the beginning of stage3 and stage4 in ShuffleNet v2. This indicates that our proposed method can be applied on the deep neural networks that contains group convolutions or channel splitting modules to obtain higher accuracy.

\begin{table}[hbp]
\caption{\textbf{ShuffleNet-V2 with Cyclic Shuffle}}
\begin{center}\label{shufflenet_v2}
\tablestyle{6.5pt}{1.2}
\begin{tabular}{c|ll|ll}
\shline
Module & \multicolumn{2}{l|}{Top1 Acc.}& \multicolumn{2}{l}{Top5 Acc.}\\
\hline
ShuffleNet v2(our impl.) &  68.71& &88.48 &\\
+cyclic shuffle &  \textbf{69.19\cgaphlp{+0.48}} & &\textbf{88.65\cgaphlp{+0.17}}&\\
\hline
\shline
% copy& More table copy$^{\mathrm{a}}$& &  \\
% \hline
% \multicolumn{4}{l}{$^{\mathrm{a}}$Sample of a Table footnote.}
\end{tabular}
\end{center}
\end{table}
\section{BitOps Definition}
Generally speaking, the number of floating-point operations (FLOPS) is the mainstream computational complexity metric. In Section 5 of the paper, we use the number of bit operations (BitOps)\cite{cai2020rethinking} to measure the  computational complexity of  quantized deep neural networks. For a convolutional layer with $t$ kernels of size $c*k*k$, let $h$ and $w$ be the height and width of the output feature map respectively. Then  the number of bit operations (BitOps) is:
\begin{equation}
    \rm{\#BitOps}=b_w\times b_{a}\times t\times c \times k \times k \times h \times w
\end{equation}
where $b_w$ and $b_{a}$ denotes the bit-with of quantized weights and activations, respectively.

\begin{table}[h]
\caption{\textbf{Quantization Results of Plain-18 on ImageNet.}}
\begin{center}
{
\tablestyle{15pt}{1.2}
\begin{tabular}{c|c|c|c}
\shline
Method&Precision&BitOps&Top1 Acc.\\
% \multirow{2}{*}{Method} & \multirow{2}{*}{Precision} & \multicolumn{2}{c}{Plain-18}\\
% \cmidrule{3-4}
\hline
Baseline&FP32&-&69.96 \\
LSQ&4b A,4b W &29.10G & 69.90 \\ 
\textbf{PalQuant} &\textbf{B=2,G=2} &\textbf{14.87G }& \textbf{70.12}\\
% \hline
% -  & FullPrecision &1.819G & 70.50 & 1.819G& 74.10 \\ 
\shline
\end{tabular}
}
\label{plain18_compare}
\end{center}
\end{table}

\section{PalQuant on CNNs without residual connections}
 
To demonstrate the general applicability of PalQuant, we conduct experiments on Plain-18 network which has the same net architecture as ResNet-18 except for residual connections. Here we re-implement the baseline Plain-18 and LSQ method under the same training settings as PalQuant. 
 Table ~\ref{plain18_compare} shows that PalQuant can achieve \textbf{+0.22}\% higher top1 accuracy than LSQ with nearly only a half of BitOps. This result means that PalQuant can also be applied to deep networks without residual connections. As such, we think PalQuant can't be seen as an extension of ShuffleNet. Besides, 
PalQuant is a quantization method that aims to deploy high-precision networks on low-precision accelerators. One key component of PalQuant is expanding feature map channels and dividing them into groups. And the proposed cyclic shuffle is another contribution of PalQuant. These two contributions make PalQuant differ a lot from ShuffleNet.

\bibliographystyle{splncs04}

\end{document}